\newif\ifJOURNAL
\JOURNALfalse
\newif\ifCONF
\CONFfalse
\newif\ifarXiv
\arXivfalse
\newif\ifWP
\WPfalse
\newif\ifFULL
\FULLfalse

\arXivtrue

\newif\ifnotCONF  %
\notCONFtrue
\ifCONF\notCONFfalse\fi

\newif\ifnotarXiv  %
\notarXivtrue
\ifarXiv\notarXivfalse\fi

\newif\ifTR  %
\TRfalse
\ifarXiv\TRtrue\fi
\ifWP\TRtrue\fi
\newif\ifnotTR
\notTRtrue
\ifarXiv\notTRfalse\fi
\ifWP\notTRfalse\fi

\ifCONF
  \newcommand{\OCMOSIX}{vovk/etal:2004NIPS}
\fi
\ifarXiv
  \newcommand{\OCMOSIX}{vovk/etal:2004NIPS}
\fi
\ifWP
  \newcommand{\OCMOSIX}{OCM9OS}
\fi

\newlength{\picturewidth}
\ifCONF
  \documentclass[]{article}
  \usepackage{proceed2e,amsmath,amsthm,amsfonts,amssymb,latexsym,graphicx,url,stmaryrd,algorithm,algorithmic}
  \usepackage{times} %
  \newcommand{\Extra}[1]{}
\fi

\ifarXiv
  \documentclass[10pt]{article}
  \usepackage{amsmath,amsthm,amsfonts,amssymb,latexsym,graphicx,url,stmaryrd,algorithm,algorithmic}
  \newcommand{\Extra}[1]{}
\fi

\ifWP
  \documentclass[10pt]{article}
  \usepackage{amsmath,amsthm,amsfonts,amssymb,latexsym,graphicx,url,stmaryrd,algorithm,algorithmic}
  \input{OT2enc.def}
  
  \usepackage{CJK}
  \input{/Doc/Computing/Latex/kp.txt}
  \newcommand{\Extra}[1]{}
  \newcommand{\zzrelax}[1]{}
\fi

\ifFULL
  \usepackage{color}
  \renewcommand{\Extra}[1]{\red{#1}}
  \newcommand{\red}[1]{\textcolor{red}{#1}}
  
  \newcommand{\bluebegin}{\begingroup\color{blue}}
  \newcommand{\blueend}{\endgroup}

\fi

\allowdisplaybreaks[4]

\DeclareMathOperator{\Prob}{\mathbb{P}}   %
\DeclareMathOperator{\Expect}{\mathbb{E}} %
\DeclareMathOperator{\conv}{conv}         %
\DeclareMathOperator{\sq}{sq}             %

\newcommand{\st}{\mathrel{\!|\!}}
\newcommand{\givn}{\mathrel{|}}

\newcommand{\bbbr}{\mathbb{R}}     %

  \newtheorem{lemma}{Lemma}
  \newtheorem{proposition}{Proposition}
  \newtheorem{corollary}{Corollary}
  \newtheorem{theorem}{Theorem}
  \theoremstyle{definition}
  \newtheorem*{remark}{Remark}

\ifCONF
  \title{Venn--Abers Predictors}
  \author{\textbf{Vladimir Vovk} and \textbf{Ivan Petej}\\
  Department of Computer Science\\
  Royal Holloway, University of London\\
  Egham, Surrey, UK}
\fi

\ifarXiv
  \title{Venn--Abers Predictors\thanks{To appear in the Proceedings of UAI 2014.}}
  \author{Vladimir Vovk and Ivan Petej\\
  \texttt{v.vovk{\rm@}rhul.ac.uk, ivan.petej{\rm@}gmail.com}}
\fi

\ifWP
  \title{Venn--Abers Predictors}
  \author{Vladimir Vovk and Ivan Petej}

  \twodatestrue
  
\fi

\begin{document}
\maketitle

\begin{abstract}
  This paper continues study, both theoretical and empirical,
  of the method of Venn prediction,
  concentrating on binary prediction problems.
  Venn predictors produce probability-type predictions
  for the labels of test objects
  which are guaranteed to be well calibrated
  under the standard assumption that the observations
  are generated independently from the same distribution.
  We give a simple formalization and proof of this property.
  We also introduce Venn--Abers predictors,
  a new class of Venn predictors based on the idea of isotonic regression,
  and report promising empirical results both for Venn--Abers predictors
  and for their more computationally efficient simplified version.
\end{abstract}

\section{Introduction}      %
\label{sec:introduction}

Venn predictors were introduced in \cite{\OCMOSIX}
and are discussed in detail in \cite{vovk/etal:2005book}, Chapter 6,
but to make the paper self-contained we define them in Section~\ref{sec:VP}.
This section also states the important property of validity of Venn predictors:
they are automatically well calibrated.
In some form this property of validity has been known:
see, e.g., \cite{vovk/etal:2005book}, Theorem~6.6.
However, this known version is complicated,
whereas our version (Theorem~\ref{thm:validity} below) is much simpler
and the intuition behind it is more transparent.
In the same section we show (Theorem~\ref{thm:uniqueness})
that Venn prediction is essentially the only way to achieve
our new property of validity.

Section~\ref{sec:VAP} defines a natural class of Venn predictors,
which we call Venn--Abers predictors
(with the ``Abers'' part formed by the initial letters of the authors' surnames
of the paper \cite{ayer/etal:1955} introducing the underlying technique).
The latter are defined on top of a wide class of classification algorithms,
which we call ``scoring classifiers'' in this paper;
each scoring classifier can be automatically transformed into a Venn--Abers predictor,
and we refer to this transformation as the ``Venn--Abers method''.
Because of its theoretical guarantees,
this method can be used for improving the calibration of probabilistic predictions.

The definition of Venn--Abers predictors was prompted by \cite{lambrou/etal:2012},
which demonstrated that the method of calibrating probabilistic predictions
introduced by Zadrozny and Elkan in \cite{zadrozny/elkan:2002}
(an adaptation of the isotonic regression procedure of \cite{ayer/etal:1955})
does not always achieve its goal and sometimes leads to poorly calibrated predictions.
Another paper reporting the possibility for the Zadrozny--Elkan method
to produce grossly miscalibrated predictions is \cite{jiang/etal:2011}.
The Venn--Abers method is a simple modification of Zadrozny and Elkan's method;
being a special case of Venn prediction,
it overcomes the problem of potentially poor calibration.

Theorem~\ref{thm:validity} in Section~\ref{sec:VP}
says that Venn predictors are perfectly calibrated.
The price to pay, however,
is that Venn predictors are multiprobabilistic predictors,
in the sense of issuing a set of probabilistic predictions
instead of a single probabilistic prediction;
intuitively, the diameter of this set reflects the uncertainty of our prediction.
In Section~\ref{sec:experiments}
we explore the efficiency of Venn--Abers predictors empirically
using the fundamental log loss function and another popular loss function,
square loss.
To apply these loss functions, we need, however,
probabilistic predictions rather than multiprobabilistic predictions,
and in Section~\ref{sec:minimax} we define natural minimax ways
of replacing the latter with the former.

In Section~\ref{sec:experiments} we explore the empirical predictive performance
of the most natural version of the original Zadrozny--Elkan method,
the Venn--Abers method, and the latter's simplified version,
which is not only simpler but also more efficient computationally.
We use nine benchmark data sets from the UCI repository \cite{UCI:data}
and six standard scoring classifiers,
and for each combination of a data set and classifier
evaluate the predictive performance of each method.
Our results show that the Venn--Abers and simplified Venn--Abers methods
usually improve the performance of the underlying classifiers,
and in our experiments they work better than the original Zadrozny--Elkan method.

Interestingly, the predictive performance of the simplified Venn--Abers method
is slightly better than that of the Venn--Abers method on the chosen data sets and scoring classifiers;
for example, in the case of the log loss function,
the best performance is achieved by the simplified Venn--Abers methods
for seven data sets out of the nine.
If these results are confirmed in wider empirical studies,
the simplified Venn--Abers method is preferred
since it achieves both computational and predictive efficiency.

Our empirical study in Section~\ref{sec:experiments} does not mean
that we recommend that the multiprobabilistic predictions output by Venn--Abers
(and more generally Venn) predictors
be replaced by probabilistic predictions
(e.g., using the formulas of Section~\ref{sec:minimax}).
On the contrary, we believe that the size of a multiprobabilistic prediction
carries valuable information about the uncertainty of the prediction.
The only purpose of replacing multiprobabilistic by probabilistic predictions
is to facilitate comparison of various prediction algorithms
using well-established loss functions.

\section{Venn predictors}
\label{sec:VP}

We consider \emph{observations} $z=(x,y)$
consisting of two components: an \emph{object} $x\in\mathbf{X}$ and its \emph{label} $y\in\mathbf{Y}$.
In this paper we are only interested in the binary case
and for concreteness set $\mathbf{Y}:=\{0,1\}$.
We assume that $\mathbf{X}$ is a measurable space,
so that observations are elements of the measurable space
that is the Cartesian product
$\mathbf{Z}:=\mathbf{X}\times\mathbf{Y}=\mathbf{X}\times\{0,1\}$.

A \emph{Venn taxonomy} $A$ is a measurable function
that assigns to each $n\in\{2,3,\ldots\}$
and each sequence $(z_1,\ldots,z_n)\in\mathbf{Z}^n$
an equivalence relation $\sim$ on $\{1,\ldots,n\}$ which is equivariant in the sense that,
for each $n$ and each permutation $\pi$ of $\{1,\ldots,n\}$,
$$
  (i\sim j \mid z_1,\ldots,z_n)
  \Longrightarrow
  (\pi(i)\sim\pi(j)\mid z_{\pi(1)},\ldots,z_{\pi(n)}),
$$
where the notation $(i\sim j \mid z_1,\ldots,z_n)$ means that $i$ is equivalent to $j$
under the relation assigned by $A$ to $(z_1,\ldots,z_n)$.
The measurability of $A$ means that for all $n$, $i$, and $j$ the set
$\{(z_1,\ldots,z_n)\st(i\sim j\mid z_1,\ldots,z_n)\}$
is measurable.
Define
\begin{equation*} %
  A(j \mid z_1,\ldots,z_n) %
  :=
  \left\{
    i\in\{1,\ldots,n\}
    \mid
    (i\sim j \mid z_1,\ldots,z_n)
  \right\}
\end{equation*}
to be the equivalence class of $j$.
Let $(z_1,\ldots,z_l)$ be a training sequence of observations $z_i=(x_i,y_i)$, $i=1,\ldots,l$,
and $x$ be a test object.
The \emph{Venn predictor} associated with a given Venn taxonomy $A$
outputs the pair $(p_0,p_1)$ as its prediction for $x$'s label,
where
$$
  p_y
  :=
  \frac
  {\left|\left\{i\in A(l+1\mid z_1,\ldots,z_l,(x,y))\mid y_i=1\right\}\right|}
  {\left|A(l+1\mid z_1,\ldots,z_l,(x,y))\right|}
$$
for both $y\in\{0,1\}$
(notice that the denominator is always positive).
Intuitively, $p_0$ and $p_1$ are the predicted probabilities that the label of $x$ is $1$;
of course, the prediction is useful only when $p_0\approx p_1$.
The \emph{probability interval} output by a Venn predictor
is defined to be the convex hull $\conv(p_0,p_1)$ of the set $\{p_0,p_1\}$;
we will sometimes refer to the pair $(p_0,p_1)$ or the set $\{p_0,p_1\}$
as the \emph{multiprobabilistic prediction}.

\subsection*{Validity of Venn predictors}

Let us say that a random variable $P$ taking values in $[0,1]$
is \emph{perfectly calibrated} for a random variable $Y$ taking values in $\{0,1\}$
if
\begin{equation}\label{eq:calibration}
  \Expect(Y\givn P)
  =
  P
  \qquad
  \text{a.s.}
\end{equation}
Intuitively, $P$ is the prediction made by a probabilistic predictor for $Y$,
and perfect calibration means that the probabilistic predictor gets the probabilities right,
at least on average,
for each value of the prediction.
A probabilistic predictor for $Y$ whose prediction $P$
satisfies (\ref{eq:calibration}) with an approximate equality
is said to be well calibrated \cite{dawid:ESS2006PF},
or unbiased in the small \cite{murphy/epstein:1967,dawid:ESS2006PF};
this terminology will be used only in informal discussions, of course.

A \emph{selector} is a random variable taking values $0$ or $1$.

\begin{theorem}\label{thm:validity}
  Let $(X_1,Y_1),\ldots,(X_l,Y_l),(X,Y)$ be IID
  (independent identically distributed) random observations.
  Fix a Venn predictor $V$ and an $l\in\{1,2,\ldots\}$.
  Let $(P_0,P_1)$ be the output of $V$ given $(X_1,Y_1,\ldots,X_l,Y_l)$
  as the training set and $X$ as the test object.
  There exists a selector $S$ such that $P_S$
  is perfectly calibrated for $Y$.
\end{theorem}

Intuitively, at least one of the two probabilities output by the Venn predictor
is perfectly calibrated.
Therefore, if the two probabilities tend to be close to each other,
we expect them (or, say, their average) to be well calibrated.

In the proof of Theorem~\ref{thm:validity} and later in the paper
we will use the notation $\lbag a_1,\ldots,a_n\rbag$ for bags (in other words, multisets);
the cardinality of the set $\{a_1,\ldots,a_n\}$ might well be smaller than $n$
(because of the removal of all duplicates in the bag).
Intuitively, $\lbag a_1,\ldots,a_n\rbag$ is the sequence $(a_1,\ldots,a_n)$
with its ordering forgotten.
We will sometimes refer to the bag $\lbag z_1,\ldots,z_l\rbag$,
where $(z_1,\ldots,z_l)$ is the training sequence,
as the training set (although technically it is a multiset rather than a set).

\begin{proof}[Proof of Theorem~\ref{thm:validity}]
  Take $S:=Y$ as the selector.
  Let us check that~(\ref{eq:calibration}) is true
  even if we further condition on the observed bag
  $\lbag (X_1,Y_1),\ldots,(X_l,Y_l),(X,Y)\rbag$
  (so that the remaining randomness consists in generating
  a random permutation of this bag).
  We only need to check the equality
  $
    \Expect(Y\givn P=p)
    =
    p
  $,
  where $P$ is the average of $1$s in the equivalence class containing $(X,Y)$,
  for the $p$s which are the percentages of $1$s in various equivalence classes
  (further conditioning on the observed bag
  is not reflected in our notation).
  For each such $p$,
  $\Expect(Y\givn P=p)$ is the average of $1$s in the equivalence classes
  for which the average of $1$s is $p$;
  therefore, we indeed have $\Expect(Y\givn P=p)=p$.
\end{proof}

The following simple corollary of Theorem~\ref{thm:validity}
gives a weaker property of validity,
which is sometimes called ``unbiasedness in the large''
\cite{murphy/epstein:1967,dawid:ESS2006PF}.
\begin{corollary}\label{cor:validity}
  For any Venn predictor $V$ and any $l=1,2,\ldots$,
  \begin{equation}\label{eq:validity}  %
    \Prob(Y=1)
    \in
    \bigl[
      \Expect
      \left(
        \underline{V}(X;X_1,Y_1,\ldots,X_l,Y_l)
      \right), %
      \Expect
      \left(
        \overline{V}(X;X_1,Y_1,\ldots,X_l,Y_l)
      \right)
    \bigr],
  \end{equation}
  where $(X_1,Y_1),\ldots,(X_l,Y_l),(X,Y)$ are IID observations
  and $[\underline{V}(\ldots),\overline{V}(\ldots)]$ is the probability interval produced by $V$
  for the test object $X$ based on the training sequence $(X_1,Y_1,\ldots,X_l,Y_l)$.
\end{corollary}
  
\begin{proof}
  It suffices to notice that,
  for a selector $S$ such that $P=P_S$ 
  ($(P_0,P_1)$ being the output of $V$)
  satisfies the condition of perfect calibration~(\ref{eq:calibration}),
  \begin{equation*}  %
    \Prob(Y=1)
    =
    \Expect(Y)
    =
    \Expect(\Expect(Y\mid P_S)) %
    =
    \Expect(P_S)
    \in
    \left[
      \Expect\underline{V},\Expect\overline{V}
    \right],
  \end{equation*}
  where the arguments of $\underline{V}$ and $\overline{V}$ are omitted.
\end{proof}

Unbiasedness in the large \eqref{eq:validity} is easy to achieve even for probabilistic predictors
if we do not care about other measures of quality of our predictions:
for example, the probabilistic predictor ignoring the $x$s
and outputting $k/l$ as its prediction,
where $k$ is the number of $1$s in the training sequence of size $l$,
is unbiased in the large.
Unbiasedness in the small~\eqref{eq:calibration} is also easy to achieve
if we allow multiprobabilistic predictors:
consider the multiprobabilistic predictor ignoring the $x$s
and outputting $\{k/(l+1),(k+1)/(l+1)\}$ as its prediction.
The problem is how to achieve predictive efficiency
(making our prediction as relevant to the test object as possible without overfitting)
while maintaining validity.

Our following result, Theorem~\ref{thm:uniqueness},
will say that under mild regularity conditions
unbiasedness in the small \eqref{eq:calibration} holds only for Venn predictors
(perhaps weakened by adding irrelevant probabilistic predictions)
and, therefore,
implies all other properties of validity,
such as the more complicated one given in~\cite[Chapter~6]{vovk/etal:2005book}.

To state Theorem~\ref{thm:uniqueness} we need a few further definitions.
Let us fix the length $l$ of the training sequence for now.
A \emph{multiprobabilistic predictor} is a function
that maps each sequence $(z_1,\ldots,z_l)\in\mathbf{Z}^l$ to a subset of $[0,1]$
(not required to be measurable in any sense).
Venn predictors are an important example for this paper.
Let us say that a multiprobabilistic predictor is \emph{invariant}
if it is independent of the ordering of the training set
$(z_1,\ldots,z_l)$.
An \emph{invariant selector} for an invariant multiprobabilistic predictor $F$
is a measurable function $f:\mathbf{Z}^{l+1}\to[0,1]$
such that $f(z_1,\ldots,z_{l+1})$ does not change
when $z_1,\ldots,z_l$ are permuted
and such that $f(z_1,\ldots,z_{l+1})\in F(z_1,\ldots,z_{l})$
for all $(z_1,\ldots,z_{l+1})$.
(It is natural to consider only invariant predictors and selectors
under the IID assumption
because of the principle of sufficiency \cite[Chap.~2]{cox/hinkley:1974}.)
We say that an invariant multiprobabilistic predictor $F$
is \emph{invariantly perfectly calibrated}
if it has an invariant selector $f$ such that
\begin{equation}\label{eq:star}  %
  \Expect
  \bigl(
    Y\mid f(Z_1,\ldots,Z_{l},(X,Y))
  \bigr) %
  =
  f(Z_1,\ldots,Z_{l},(X,Y))
  \enspace
  \text{a.s.}
\end{equation}
whenever $Z_1,\ldots,Z_l,(X,Y)$ are IID observations.

\begin{theorem}\label{thm:uniqueness}
  If an invariant multiprobabilistic predictor $F$ is invariantly perfectly calibrated,
  then it contains a Venn predictor $V$ in the sense that
  both elements of $V(Z_1,\ldots,Z_{l})$ belong to $F(Z_1,\ldots,Z_{l})$
  almost surely provided $Z_1,\ldots,Z_l$ are IID.
\end{theorem}

\begin{proof}
  Let $f$ be an invariant selector of $F$
  satisfying the condition~(\ref{eq:star})
  of being invariantly perfectly calibrated.
  By definition,
  \begin{equation*}
    \Expect
    \bigl(
      Y-f(Z_1,\ldots,Z_{l},(X,Y))
      \mid{} %
      f(Z_1,\ldots,Z_{l},(X,Y))
    \bigr)
    =
    0
    \enspace
    \text{a.s.},
  \end{equation*}
  which implies
  \begin{equation}\label{eq:conditional}  %
    \Expect
    \bigl(
      (Y-f(Z_1,\ldots,Z_{l},(X,Y))) %
      1_{\{f(Z_1,\ldots,Z_{l},(X,Y))\in[a,b]\}}
    \bigr)
    =
    0
    \enspace
    \text{a.s.}
  \end{equation}
  for all intervals $[a,b]$ with rational end-points.
  The expected value in \eqref{eq:conditional}
  can be obtained in two steps:
  first we average
  \begin{equation*}
    (y'_{l+1}-f(z'_1,\ldots,z'_{l+1}))
    1_{\{f(z'_1,\ldots,z'_{l+1})\in[a,b]\}}
  \end{equation*}
  over the orderings $(z'_1,\ldots,z'_{l+1})$
  of each bag $\lbag z_1,\ldots,z_{l+1}\rbag$,
  where $z_i=(x_i,y_i)$ and $z'_i=(x'_i,y'_i)$,
  and then we average over the bags $\lbag z_1,\ldots,z_{l+1}\rbag$
  generated according to $z_i:=Z_i$, $i=1,\ldots,l$, and $z_{l+1}:=(X,Y)$.
  The first operation is discrete:
  the average over the orderings of $\lbag z_1,\ldots,z_{l+1}\rbag$
  is the arithmetic mean of $(y_i-p_i)1_{\{p_i\in[a,b]\}}$ over $i=1,\ldots,l+1$,
  where $p_i:=f(\ldots,z_i)$ and the dots stand
  for $z_1,\ldots,z_{i-1}$ and $z_{i+1},\ldots,z_{l+1}$ arranged in any order
  (since $f$ is invariant, the order does not matter).
  By the completeness of the statistic that maps a data sequence of size $l+1$
  to the corresponding bag \cite[Section~4.3]{lehmann:1986},
  this average is zero for all $[a,b]$ and almost all bags.
  Without loss of generality we assume that this holds for all bags.

  Define a Venn taxonomy $A$ as follows:
  given a sequence $(z_1,\ldots,z_{l+1})$,
  set $i\sim j$ if $p_i=p_j$ where $p$ is defined as above.
  It is easy to check that the corresponding Venn predictor
  satisfies the requirement in Theorem~\ref{thm:uniqueness}.
\end{proof}

\begin{remark}
  The invariance assumption in Theorem~\ref{thm:uniqueness}
  is essential.
  Indeed, suppose $l>1$ and consider the multiprobabilistic predictor
  whose prediction for the label of the test observation
  does not depend on the objects and is
  $$
    \begin{cases}
      \{k/l,(k+1)/l\} & \text{if $y_1=0$}\\
      \{(k-1)/l,k/l\} & \text{if $y_1=1$},
    \end{cases}
  $$
  where $k$ is the number of 1s among the labels of the $l$ training observations.
  This non-invariant predictor is perfectly calibrated
  (see below)
  but does not contain a Venn predictor
  (if it did, such a Venn predictor, being invariant,
  would always output the one-element multiprobabilistic prediction $\{k/l\}$,
  which is impossible).
    Let us check that this non-invariant predictor is indeed perfectly calibrated,
    even given the union of the training set and the test observation
    (i.e., given the bag of size $l+1$ obtained from the training sequence
    by joining the test observation and then forgetting the ordering).
    Take the selector such that the selected probabilistic predictor is
    $$
      \begin{cases}
        k/l & \text{for sequences of the form $0\ldots0$}\\
        (k+1)/l & \text{for sequences of the form $0\ldots1$}\\
        (k-1)/l & \text{for sequences of the form $1\ldots0$}\\
        k/l & \text{for sequences of the form $1\ldots1$}.
      \end{cases}
    $$
    For a binary sequence of labels of length $l+1$ with $m$ 1s
    the probabilistic prediction $P$
    for its last element will be, therefore,
    $$
      \begin{cases}
        m/l & \text{for sequences of the form $0\ldots0$}\\
        m/l & \text{for sequences of the form $0\ldots1$}\\
        (m-1)/l & \text{for sequences of the form $1\ldots0$}\\
        (m-1)/l & \text{for sequences of the form $1\ldots1$}.
      \end{cases}
    $$
    The conditional probability that $Y=1$
    ($Y$ being the label of the last element)
    given $P=p$ (and given $m$) is
    $$
      \frac{\binom{l-1}{m-1}}{\binom{l}{m}}
      =
      \frac{m}{l}
    $$
    when $p=m/l$
    and is
    $$
      \frac{\binom{l-1}{m-2}}{\binom{l}{m-1}}
      =
      \frac{m-1}{l}
    $$
    when $p=(m-1)/l$;
    in both cases we have perfect calibration.
\end{remark}

\section{Venn--Abers predictors}
\label{sec:VAP}

We say that a function $f$ is \emph{increasing}
if its domain is an ordered set and $t_1\le t_2\Rightarrow f(t_1)\le f(t_2)$.

Many machine-learning algorithms for classification are in fact \emph{scoring classifiers}:
when trained on a training sequence of observations and fed with a test object $x$,
they output a \emph{prediction score} $s(x)$;
we will call $s:\mathbf{X}\to\bbbr$ the \emph{scoring function} for that training sequence.
The actual classification algorithm is obtained by fixing a threshold $c$
and predicting the label of $x$ to be $1$ if and only if $s(x)\ge c$
(or if and only if $s(x)>c$).
Alternatively, one could apply an increasing function $g$ to $s(x)$
in an attempt to ``calibrate'' the scores,
so that $g(s(x))$ can be used as the predicted probability that the label of $x$ is $1$.

Fix a scoring classifier and let $(z_1,\ldots,z_l)$ be a training sequence
of observations $z_i=(x_i,y_i)$, $i=1,\ldots,l$.
The most direct application \cite{zadrozny/elkan:2002} of the method of isotonic regression
\cite{ayer/etal:1955} to the problem of score calibration is as follows.
Train the scoring classifier on the training sequence
and compute the score $s(x_i)$ for each training observation $(x_i,y_i)$,
where $s$ is the scoring function for $(z_1,\ldots,z_l)$.
Let $g$ be the increasing function on the set $\{s(x_1),\ldots,s(x_l)\}$
that maximizes the likelihood
\begin{equation}\label{eq:likelihood}
  \prod_{i=1}^l
  p_i,
  \text{\quad where }
  p_i
  :=
  \begin{cases}
    g(s(x_i)) & \text{if $y_i=1$}\\
    1-g(s(x_i)) & \text{if $y_i=0$}.
  \end{cases}
\end{equation}
Such a function $g$ is indeed unique
\cite[Corollary~2.1]{ayer/etal:1955}
and can be easily found using the ``pair-adjacent violators algorithm''
(PAVA, described in detail in the summary of \cite{ayer/etal:1955}
and in \cite[Section~1.2]{barlow/etal:1972};
see also the proof of Lemma~\ref{lem:must-be-well-known} below).
We will say that $g$ is the \emph{isotonic calibrator}
for $((s(x_1),y_1),\ldots,(s(x_l),y_l))$.
To predict the label of a test object $x$,
the direct method finds the closest $s(x_i)$ to $s(x)$ and outputs $g(s(x_i))$ as its prediction
(in the case of ties our implementation of this method used in Section~\ref{sec:experiments}
chooses the smaller $s(x_i)$;
however, ties almost never happen in our experiments).
We will refer to this as the \emph{direct isotonic-regression} (DIR) method.

The direct method is prone to overfitting
as the same observations $z_1,\ldots,z_l$ are used both for training the scoring classifier
and for calibration without taking any precautions.
The \emph{Venn--Abers predictor} corresponding to the given scoring classifier
is the multiprobabilistic predictor that is defined as follows.
Try the two different labels, $0$ and $1$, for the test object $x$.
Let $s_0$ be the scoring function for $(z_1,\ldots,z_l,(x,0))$,
$s_1$ be the scoring function for $(z_1,\ldots,z_l,(x,1))$,
$g_0$ be the isotonic calibrator for
\begin{equation}\label{eq:ic-1}
  \bigl(
    (s_0(x_1),y_1),\ldots,(s_0(x_l),y_l),(s_0(x),0)
  \bigr),
\end{equation}
and $g_1$ be the isotonic calibrator for
\begin{equation}\label{eq:ic-2}
  \bigl(
    (s_1(x_1),y_1),\ldots,(s_1(x_l),y_l),(s_1(x),1)
  \bigr).
\end{equation}
The multiprobabilistic prediction output by the Venn--Abers predictor
is $(p_0,p_1)$,
where $p_0:=g_0(s_0(x))$ and $p_1:=g_1(s_1(x))$.
(And we can expect $p_0$ and $p_1$ to be close to each other
unless DIR overfits grossly.)
The Venn--Abers predictor is described as Algorithm~\ref{alg:VA}.

\begin{algorithm}[bt]
  \caption{Venn--Abers predictor}
  \label{alg:VA}
  \begin{algorithmic}
    \renewcommand{\algorithmicrequire}{\textbf{Input:}}
    \renewcommand{\algorithmicensure}{\textbf{Output:}}
    \REQUIRE training sequence $(z_1,\ldots,z_l)$
    \REQUIRE test object $x$
    \ENSURE multiprobabilistic prediction $(p_0,p_1)$ %
    \FOR{$y\in\{0,1\}$}
      \STATE set $s_y$ to the scoring function for
        $(z_1,\ldots,z_l,(x,y))$
      \STATE set $g_y$ to the isotonic calibrator for
        \STATE \qquad $(s_y(x_1),y_1),\ldots,(s_y(x_l),y_l),(s_y(x),y)$
      \STATE set $p_y:=g_y(s_y(x))$
    \ENDFOR
  \end{algorithmic}
\end{algorithm}

The intuition behind Algorithm~\ref{alg:VA}
is that it tries to evaluate the robustness of the DIR prediction.
To see how sensitive the scoring function is to the training set
we extend the latter by adding the test object labelled in two different ways.
And to see how sensitive the probabilistic prediction is,
we again consider the training set extended in two different ways
(if it is sensitive, the prediction will be fragile
even if the scoring function is robust).
For large data sets and inflexible scoring functions,
we will have $p_0\approx p_1$,
and both numbers will be close to the DIR prediction.
However, even if the data set is very large but the scoring function is very flexible,
$p_0$ can be far from $p_1$
(the extreme case is where the scoring function is so flexible that it ignores all observations
apart from a few that are most similar to the test object,
and in this case it does not matter how big the data set is).
We rarely know in advance how flexible our scoring function is
relative to the size of the data set,
and the difference between $p_0$ and $p_1$ gives us some indication of this.

The following proposition says that Venn--Abers predictors are Venn predictors
and, therefore, inherit all properties of validity of the latter,
such as Theorem~\ref{thm:validity}.

\begin{proposition}\label{prop:VAP}
  Venn--Abers predictors are Venn predictors.
\end{proposition}

\begin{proof}
  Fix a Venn--Abers predictor.
  The corresponding Venn taxonomy is defined as follows:
  given a sequence
  $$
    (z_1,\ldots,z_n)
    =
    ((x_1,y_1),\ldots,(x_n,y_n))
    \in
    (\mathbf{X}\times\{0,1\})^n
  $$
  and $i,j\in\{1,\ldots,n\}$,
  we set $i\sim j$ if and only if $g(s(x_i))=g(s(x_j))$,
  where $s$ is the scoring function for $(z_1,\ldots,z_n)$
  and $g$ is the isotonic calibrator for
  $$
    \bigl(
      (s(x_1),y_1),\ldots,(s(x_n),y_n)
    \bigr).
  $$
  Lemma~\ref{lem:must-be-well-known} below shows
  that the Venn predictor corresponding to this taxonomy
  gives predictions identical to those given by the original Venn--Abers predictor.
  This proves the proposition.
\end{proof}

\begin{lemma}\label{lem:must-be-well-known}
  Let $g$ be the isotonic calibrator for
  $((t_1,y_1),\ldots,(t_n,y_n))$,   %
  where $t_i\in\bbbr$ and $y_i\in\{0,1\}$, $i=1,\ldots,n$.
  Any $p\in\{g(t_1),\ldots,g(t_n)\}$ is equal to the arithmetic mean
  of the labels $y_i$ of the $t_i$, $i=1,\ldots,n$, satisfying $g(t_i)=p$.
\end{lemma}

\begin{proof}
  The statement of the lemma immediately follows from the definition of the PAVA
  \cite[summary]{ayer/etal:1955}, which we will reproduce here.
  Arrange the numbers $t_i$ in the strictly increasing order
  $t_{(1)}<\cdots<t_{(k)}$, where $k\le n$ is the number of distinct elements among $t_i$.
  We would like to find the increasing function $g$
  on the set $\{t_{(1)},\ldots,t_{(k)}\}=\{t_1,\ldots,t_n\}$
  maximizing the likelihood
  (defined by (\ref{eq:likelihood}) with $t_i$ in place of $s(x_i)$ and $n$ in place of $l$).
  The procedure is recursive.
  At each step the set $\{t_{(1)},\ldots,t_{(k)}\}$
  is partitioned into a number of disjoint cells consisting of adjacent elements of the set;
  to each cell is assigned a ratio $a/N$ (formally, a pair of integers, with $a\ge0$ and $N>0$);
  the function $g$ defined at this step (perhaps to be redefined at the following steps)
  is constant on each cell.
  For $j=1,\ldots,k$,
  let $a_{(j)}$ be the number of $i$ such that $y_i=1$ and $t_i=t_{(j)}$,
  and let $N_{(j)}$ be the number of $i$ such that $t_i=t_{(j)}$.
  Start from the partition of $\{t_{(1)},\ldots,t_{(k)}\}$ into one-element cells,
  assign the ratio $a_{(j)}/N_{(j)}$ to $\{t_{(j)}\}$,
  and set
  \begin{equation}\label{eq:initial}
    g(t_{(j)}):=\frac{a_{(j)}}{N_{(j)}}
  \end{equation}
  (in the notation used in this proof,
  $a/N$ is a pair of integers whereas $\frac{a}{N}$ is a rational number, the result of the division).
  If the function $g$ is increasing, we are done.
  If not, there is a pair $C_1,C_2$ of adjacent cells (``violators'')
  such that $C_1$ is to the left of $C_2$ and $g(C_1)>g(C_2)$
  (where $g(C)$ stands for the common value of $g(t_{(j)})$ for $t_{(j)}\in C$);
  in this case redefine the partition by merging $C_1$ and $C_2$ into one cell $C$,
  assigning the ratio $(a_1+a_2)/(N_1+N_2)$ to $C$,
  where $a_1/N_1$ and $a_2/N_2$ are the ratios assigned to $C_1$ and $C_2$, respectively,
  and setting
  \begin{equation}\label{eq:step}  %
    g(t_{(j)})
    :=
    \frac{N_1}{N_1+N_2}
    g(C_1)
    +
    \frac{N_2}{N_1+N_2}
    g(C_2) %
    =
    \frac{a_1+a_2}{N_1+N_2}
  \end{equation}   %
  for all $t_{(j)}\in C$.
  Repeat the process until $g$ becomes increasing
  (the number of cells decreases by $1$ at each iteration,
  so the process will terminate in at most $k$ steps).
  The final function $g$ is the one that maximizes the likelihood.
  The statement of the lemma follows from this recursive definition:
  it is true by definition for the initial function (\ref{eq:initial})
  and remains true when $g$ is redefined by (\ref{eq:step}).
\end{proof}

\section{Probabilistic predictors derived from Venn predictors}
\label{sec:minimax}

In the next section we will compare Venn--Abers predictors with known probabilistic predictors
using standard loss functions.
Since Venn--Abers predictors output pairs of probabilities rather than point probabilities,
we will need to fit them (somewhat artificially) in the standard framework
extracting one probability $p$ from $p_0$ and $p_1$.

In this paper we will use two loss functions, log loss and square loss.
The \emph{log loss} suffered when predicting $p\in[0,1]$ whereas the true label
is $y$ is
$$
  \lambda_{\ln}(p,y)
  :=
  \begin{cases}
    -\ln(1-p) & \text{if $y=0$}\\
    -\ln p & \text{if $y=1$}.
  \end{cases}
$$
This is the most fundamental loss function,
since the cumulative loss $\sum_{i=1}^n\lambda_{\ln}(p_i,y_i)$
over a test sequence of size $n$
is equal to the minus log of the probability that the predictor assigns
to the sequence
(this assumes either the batch mode of prediction with independent test observations
or the online mode of prediction);
therefore, a smaller cumulative log loss corresponds to a larger probability.
The \emph{square loss} suffered when predicting $p\in[0,1]$
for the true label $y$ is
$$
  \lambda_{\sq}(p,y)
  :=
  (y-p)^2.
$$
The main advantage of this loss function is that it is \emph{proper}
(see, e.g., \cite{dawid:ESS2006PF}):
the function $\Expect_{y\sim B_p}\lambda_{\sq}(q,y)$ of $q\in[0,1]$,
where $B_p$ is the Bernoulli distribution with parameter $p$,
attains its minimum at $q=p$.
(Of course, the log loss function is also proper.)

First suppose that our loss function is $\lambda_{\ln}$
and we are given a multiprobabilistic prediction $(p_0,p_1)$;
let us find the corresponding minimax probabilistic prediction $p$.
If the true outcome is $y=0$,
our regret for using $p$ instead of the appropriate $p_0$ is $-\ln(1-p)+\ln(1-p_0)$.
If $y=1$, our regret for using $p$ instead of the appropriate $p_1$ is $-\ln p+\ln p_1$.
The first regret as a function of $p\in[0,1]$
strictly increases from a nonpositive value to $\infty$
as $p$ changes from $0$ to $1$.
The second regret as a function of $p$ strictly decreases from $\infty$ to a nonpositive value
as $p$ changes from $0$ to $1$.
Therefore, the minimax regret is the solution to
$$
  -\ln(1-p)+\ln(1-p_0)
  =
  -\ln p+\ln p_1,
$$
which is
\begin{equation}\label{eq:log}
  p
  =
  \frac{p_1}{1-p_0+p_1}.
\end{equation}
The intuition behind this minimax value of $p$
is that we can interpret the multiprobabilistic prediction $(p_0,p_1)$
as the unnormalized probability distribution $P$ on $\{0,1\}$
such that $P(\{0\})=1-p_0$ and $P(\{1\})=p_1$;
we then normalize $P$ to get a genuine probability distribution $P':=P/P(\{0,1\})$,
and the $p$ in (\ref{eq:log}) is equal to $P'(\{1\})$.
Of course, it is always true that $p\in\conv(p_0,p_1)$.

In the case of the square loss function,
the regret is
$$
  \begin{cases}
    p^2-p_0^2 & \text{if $y=0$}\\
    (1-p)^2-(1-p_1)^2 & \text{if $y=1$}
  \end{cases}
$$
and the two regrets are equal when
\begin{equation}\label{eq:square}
  p
  :=
  p_1 + p_0^2/2 - p_1^2/2.
\end{equation}
To see how natural this expression is notice that (\ref{eq:square}) is equivalent to
$$
  p
  =
  \bar p
  +
  (p_1-p_0)
  \left(
    \frac12 - \bar p
  \right),
$$
where $\bar p:=(p_0+p_1)/2$.
Therefore, $p$ is a regularized version of $\bar p$:
we move $\bar p$ towards the neutral value $1/2$
in the typical (for the Venn--Abers method) case where $p_0<p_1$.
In any case, we always have $p\in\conv(p_0,p_1)$.

The following lemma shows that log loss is never infinite for probabilistic predictors
derived from Venn predictors.
\begin{lemma}\label{lem:finite-log}
  Neither of the methods discussed in this section
  (see~\eqref{eq:log} and~\eqref{eq:square})
  ever produces $p\in\{0,1\}$
  when applied to Venn--Abers predictors.
\end{lemma}
\begin{proof}
  Lemma~\ref{lem:must-be-well-known} implies that $p_0<1$ and that $p_1>0$.
  It remains to notice that both (\ref{eq:log}) and (\ref{eq:square})
  produce $p$ in the interior of $\conv(p_0,p_1)$ if $p_0\ne p_1$
  and produce $p=p_0=p_1$ if $p_0=p_1$
  (and this is true for any sensible averaging method).
\end{proof}

\section{Experimental results}
\label{sec:experiments}

\begin{table*}[t]
\caption{Log loss (MLE) results obtained using standard Weka classifiers (W)
  and the three calibration methods (VA, SVA, DIR) applied to the standard classifiers' outputs
  for the following Weka classifiers:
  J48, J48 Bagging, logistic regression (upper part)
  and na\"{\i}ve Bayes, neural networks, and SVM Platt (lower part).
  The best results for each pair (classifier,\,dataset) are in bold.}
\label{tab:log}
\medskip\hspace*{-1.6cm}   
{\small\begin{tabular}{l|cccc|cccc|cccc}
\multicolumn{1}{c}{}
  & \multicolumn{4}{c}{J48 (J)}
  & \multicolumn{4}{c}{J48 Bagging (JB)}
  & \multicolumn{4}{c}{logistic regression (LR)}\\
& W & VA & SVA & DIR & W & VA & SVA & DIR & W & VA & SVA & DIR\\
\hline
Australian & $\infty$ & \textbf{0.380} & 0.469 & $\infty$
  & \textbf{0.328} & 0.369 & 0.344 & $\infty$
  & 0.342 & \textbf{0.340} & 0.340 & $\infty$\\
Breast & $\infty$ & \textbf{0.607} & 0.642 & $\infty$
  & \textbf{0.581} & 0.592 & 0.636 & $\infty$
  & 0.584 & \textbf{0.567} & 0.586 & $\infty$\\
Diabetes & $\infty$ & \textbf{0.552} & 0.635 & $\infty$
  & \textbf{0.504} & 0.515 & 0.561 & $\infty$
  & 0.492 & \textbf{0.490} & 0.491 & $\infty$\\
Echo & $\infty$ & \textbf{0.606} & 0.670 & $\infty$
  & 0.556 & \textbf{0.517} & 0.563 & $\infty$
  & $\infty$ & \textbf{0.589} & 0.606 & $\infty$\\
Hepatitis & $\infty$ & \textbf{0.491} & 0.528 & $\infty$
  & \textbf{0.420} & 0.456 & 0.434 & $\infty$
  & $\infty$ & \textbf{0.393} & 0.504 & $\infty$\\
Ionosphere & $\infty$ & \textbf{0.383} & 0.410 & $\infty$
  & $\infty$ & 0.387 & \textbf{0.251} & $\infty$
  & $\infty$ & \textbf{0.387} & 0.524 & $\infty$\\
Labor & $\infty$ & \textbf{0.503} & 0.537 & $\infty$
  & 0.427 & 0.427 & \textbf{0.385} & $\infty$
  & 1.927 & 0.687 & \textbf{0.297} & $\infty$\\
Liver & $\infty$ & \textbf{0.662} & 0.866 & $\infty$
  & \textbf{0.609} & 0.635 & 0.707 & $\infty$
  & 0.619 & 0.622 & \textbf{0.611} & $\infty$\\
Vote & $\infty$ & \textbf{0.134} & 0.145 & $\infty$
  & 0.135 & 0.159 & \textbf{0.131} & $\infty$
  & 1.059 & 0.188 & \textbf{0.148} & $\infty$
\end{tabular}}
\\[2mm]\hspace*{-1.6cm}  
{\small\begin{tabular}{l|cccc|cccc|cccc}
\multicolumn{1}{c}{}
 & \multicolumn{4}{c}{na\"{\i}ve Bayes (NB)}
 & \multicolumn{4}{c}{neural networks (NN)}
 & \multicolumn{4}{c}{SVM Platt (SVM)}\\
 & W & VA & SVA & DIR & W & VA & SVA & DIR & W & VA & SVA & DIR\\
\hline
Australian & 0.839 & \textbf{0.355} & 0.367 & $\infty$
  & 0.557 & \textbf{0.427} & 0.450 & $\infty$
  & 0.391 & 0.356 & \textbf{0.351} & $\infty$ \\
Breast & 0.663 & 0.563 & \textbf{0.551} & $\infty$
  & 0.774 & \textbf{0.615} & 0.738 & $\infty$
  & 0.583 & \textbf{0.568} & 0.582 & $\infty$ \\
Diabetes & 0.753 & \textbf{0.495} & 0.508 & $\infty$
  & 0.536 & \textbf{0.500} & 0.519 & $\infty$
  & 0.491 & 0.497 & \textbf{0.490} & $\infty$ \\
Echo & 0.658 & \textbf{0.505} & 0.522 & $\infty$
  & 0.770 & \textbf{0.578} & 0.605 & $\infty$
  & 0.558 & \textbf{0.495} & 0.538 & $\infty$ \\
Hepatitis & 0.936 & \textbf{0.365} & 0.372 & $\infty$
  & 0.753 & \textbf{0.471} & 0.484 & $\infty$
  & 0.435 & \textbf{0.349} & 0.404 & $\infty$ \\
Ionosphere & 0.704 & 0.262 & \textbf{0.227} & $\infty$
  & 0.625 & 0.427 & \textbf{0.379} & $\infty$
  & 0.359 & \textbf{0.250} & 0.333 & $\infty$ \\
Labor & 1.854 & 0.410 & \textbf{0.296} & $\infty$
  & 0.325 & 0.560 & \textbf{0.298} & $\infty$
  & 3.643 & 0.364 & \textbf{0.287} & $\infty$ \\
Liver & 0.727 & 0.649 & 0.661 & $\infty$
  & 0.642 & \textbf{0.603} & 0.615 & $\infty$
  & 0.645 & 0.663 & \textbf{0.639} & $\infty$ \\
Vote & 0.594 & 0.218 & \textbf{0.211} & $\infty$
  & 0.235 & 0.229 & \textbf{0.158} & $\infty$
  & 0.125 & 0.211 & \textbf{0.121} & $\infty$
\end{tabular}}
\end{table*}

\begin{table*}[t]
\caption{The analogue of Table~\ref{tab:log} for square loss (RMSE).}
\label{tab:sq}
\medskip\hspace*{-1.6cm}   
{\small\begin{tabular}{l|cccc|cccc|cccc}
\multicolumn{1}{c}{}
  & \multicolumn{4}{c}{J48 (J)}
  & \multicolumn{4}{c}{J48 Bagging (JB)}
  & \multicolumn{4}{c}{logistic regresion (LR)} \\
& W & VA & SVA & DIR & W & VA & SVA & DIR & W & VA & SVA & DIR \\
\hline
Australian & 0.366 & \textbf{0.346} & 0.359 & 0.366
  & \textbf{0.313} & 0.338 & 0.318 & 0.323
  & \textbf{0.317} & 0.319 & 0.319 & 0.321 \\
Breast & 0.472 & \textbf{0.453} & 0.463 & 0.473
  & \textbf{0.443} & 0.451 & 0.460 & 0.474
  & 0.442 & \textbf{0.437} & 0.444 & 0.450 \\
Diabetes & 0.449 & \textbf{0.431} & 0.443 & 0.449
  & \textbf{0.407} & 0.415 & 0.420 & 0.427
  & \textbf{0.399} & 0.401 & 0.401 & 0.402 \\
Echo & 0.478 & \textbf{0.456} & 0.460 & 0.482
  & 0.427 & \textbf{0.417} & 0.423 & 0.444
  & 0.457 & \textbf{0.443} & 0.446 & 0.475 \\
Hepatitis & 0.407 & \textbf{0.393} & 0.401 & 0.419
  & \textbf{0.362} & 0.390 & 0.368 & 0.391
  & 0.400 & \textbf{0.357} & 0.384 & 0.411 \\
Ionosphere & 0.318 & 0.355 & \textbf{0.312} & 0.318
  & 0.267 & 0.356 & \textbf{0.261} & 0.267
  & 0.357 & 0.363 & \textbf{0.349} & 0.361 \\
Labor & 0.407 & 0.403 & \textbf{0.402} & 0.413
  & 0.361 & 0.371 & \textbf{0.339} & 0.341
  & 0.294 & 0.498 & \textbf{0.287} & 0.303 \\
Liver & 0.528 & \textbf{0.482} & 0.518 & 0.528
  & \textbf{0.457} & 0.478 & 0.478 & 0.493
  & 0.460 & 0.463 & \textbf{0.458} & 0.461 \\
Vote & 0.187 & 0.186 & \textbf{0.186} & 0.187
  & 0.187 & 0.206 & \textbf{0.186} & 0.188
  & 0.198 & 0.233 & \textbf{0.195} & 0.203 \\
\hline
\end{tabular}}
\\[2mm]\hspace*{-1.6cm}  
{\small\begin{tabular}{l|cccc|cccc|cccc}
\multicolumn{1}{c}{}
  & \multicolumn{4}{c}{na\"{\i}ve Bayes (NB)}
  & \multicolumn{4}{c}{neural networks (NN)}
  & \multicolumn{4}{c}{SVM Platt (SVM)} \\
& W & VA & SVA & DIR & W & VA & SVA & DIR & W & VA & SVA & DIR \\
\hline
Australian & 0.392 & \textbf{0.328} & 0.333 & 0.335
  & \textbf{0.360} & 0.363 & 0.361 & 0.371
  & 0.343 & \textbf{0.324} & 0.325 & 0.327 \\
Breast & 0.449 & 0.436 & \textbf{0.427} & 0.433
  & 0.485 & \textbf{0.465} & 0.491 & 0.508
  & 0.443 & \textbf{0.431} & 0.442 & 0.447 \\
Diabetes & 0.420 & \textbf{0.406} & 0.410 & 0.413
  & 0.413 & \textbf{0.408} & 0.413 & 0.417
  & 0.399 & \textbf{0.393} & 0.400 & 0.402 \\
Echo & 0.428 & \textbf{0.408} & 0.412 & 0.426
  & 0.457 & \textbf{0.436} & 0.443 & 0.468
  & \textbf{0.416} & 0.427 & 0.418 & 0.431 \\
Hepatitis & 0.357 & 0.339 & \textbf{0.335} & 0.342
  & 0.396 & 0.402 & \textbf{0.379} & 0.427
  & 0.350 & \textbf{0.350} & 0.353 & 0.364 \\
Ionosphere & 0.281 & 0.273 & \textbf{0.250} & 0.251
  & 0.321 & 0.378 & \textbf{0.316} & 0.333
  & 0.312 & \textbf{0.309} & 0.312 & 0.316 \\
Labor & \textbf{0.256} & 0.363 & 0.284 & 0.281
  & \textbf{0.279} & 0.442 & 0.293 & 0.307
  & \textbf{0.274} & 0.358 & 0.280 & 0.283 \\
Liver & 0.480 & \textbf{0.476} & 0.478 & 0.487
  & 0.459 & \textbf{0.456} & 0.456 & 0.463
  & 0.473 & 0.477 & \textbf{0.472} & 0.477 \\
Vote & 0.292 & 0.257 & 0.251 & \textbf{0.250}
  & 0.216 & 0.271 & \textbf{0.206} & 0.227
  & \textbf{0.183} & 0.191 & 0.185 & 0.188 \\
\end{tabular}}
\end{table*}

In this section we compare various calibration methods discussed so far
by applying them to six standard scoring classifiers
(we will usually omit ``scoring'' in this section)
available within Weka \cite{weka:2011},
a machine learning tool developed at the University of Waikato, NZ.
The standard classifiers are J48 decision trees (abbreviated to J48, or even to J),
J48 decision trees with bagging (J48 Bagging, or JB),
logistic regression (LR), na\"{\i}ve Bayes (NB), neural networks (NN),
and support vector machines calibrated using a sigmoid function
as defined by Platt \cite{platt:2000} (SVM Platt, or simply SVM).
Each of these standard classifiers produces scores in the interval $[0,1]$,
which can then be used as probabilistic predictions;
however, in most previous studies these have been found to be inaccurate
(see \cite{zadrozny/elkan:2002} and \cite{langford/zadrozny:2005}).
We use the scores generated by classifiers as inputs,
and by applying the DIR (defined in Section~\ref{sec:VAP}),
Venn--Abers (VA), and simplified Venn--Abers (SVA, see below) methods
we investigate how well we can calibrate the scores
and improve them in their role as probabilistic predictions.

In the set of experiments described in this section
we do not perform a direct comparison to the method developed by Langford and Zadrozny
\cite{langford/zadrozny:2005}
primarily because, as far as we are aware,
the algorithms described in their work are not publicly available.

For the purpose of comparison we use a total of nine datasets with binary labels (encoded as 0 or 1)
obtained from the UCI machine learning repository \cite{UCI:data}:
Australian Credit (which we abbreviate to Australian),
Breast Cancer (Breast), Diabetes, Echocardiogram (Echo), Hepatitis,
Ionosphere, Labor Relations (Labor), Liver Disorders (Liver), and Congressional Voting (Vote).
The datasets vary in size as well as the number and type of attributes
in order to give a reasonable range of conditions encountered in practice.

In our comparison we use the two standard loss functions
discussed in the previous section.
Namely, on a given test sequence of length $n$
we will calculate the \emph{mean log error} (MLE)
\begin{equation}\label{eq:MLE}
  \frac1n\sum_{i=1}^n\lambda_{\ln}(p_i,y_i)
\end{equation}
and the \emph{root mean square error} (RMSE)
\begin{equation}\label{eq:RMSE}
  \sqrt{\frac1n\sum_{i=1}^n\lambda_{\sq}(p_i,y_i)},
\end{equation}
where $p_i$ is the probabilistic prediction for the label $y_i$
of the $i$th observation in the test sequence.
MLE \eqref{eq:MLE} can be infinite,
namely when predicting $p_i\in\{0,1\}$ while being incorrect.
It therefore penalises the overly confident probabilistic predictions
much more significantly than RMSE.
We compare the performance of the standard classifiers
with their versions calibrated using the three methods (VA, SVA, and DIR)
under both loss functions for each dataset.
In each experiment we randomly permute the dataset
and use the first $2/3$ observations for training and the remaining $1/3$ for testing.

One of the potential drawbacks of the Venn--Abers method is its computational inefficiency:
for each test object the scores have to be recalculated
for the training sequence extended by including the test object
first labelled as 0 and then labelled as 1.
This implies that the total calculation time
is at least $2n$ times that of the underlying classifier,
where $n$ is the number of test observations.
Therefore, we define a simplified version of Venn--Abers predictors,
for which the scores are calculated only once
without recalculating them for each test object with postulated labels~0 and~1.

In detail, the \emph{simplified Venn--Abers predictor} for a given scoring classifier
is defined as follows.
Let $(z_1,\ldots,z_l)$ be a training sequence
and $x$ be a test object.
Define $s$ to be the scoring function for $(z_1,\ldots,z_l)$,
$g_0$ to be the isotonic calibrator for
\begin{equation*}
  \bigl(
    (s(x_1),y_1),\ldots,(s(x_l),y_l),(s(x),0)
  \bigr),
\end{equation*}
and $g_1$ to be the isotonic calibrator for
\begin{equation*}
  \bigl(
    (s(x_1),y_1),\ldots,(s(x_l),y_l),(s(x),1)
  \bigr)
\end{equation*}
(cf.~\eqref{eq:ic-1} and~\eqref{eq:ic-2}).
The multiprobabilistic prediction output for the label of $x$
by the simplified Venn--Abers (SVA) predictor
is $(p_0,p_1)$, where $p_0:=g_0(s(x))$ and $p_1:=g_1(s(x))$.
This method, summarized as Algorithm~\ref{alg:SVA},
is intermediate between DIR and the Venn--Abers method.

\begin{algorithm}[bt]
  \caption{Simplified Venn--Abers predictor}
  \label{alg:SVA}
  \begin{algorithmic}
    \renewcommand{\algorithmicrequire}{\textbf{Input:}}
    \renewcommand{\algorithmicensure}{\textbf{Output:}}
    \REQUIRE training sequence $(z_1,\ldots,z_l)$
    \REQUIRE test object $x$
    \ENSURE multiprobabilistic prediction $(p_0,p_1)$ %
    \FOR{$y\in\{0,1\}$}
      \STATE set $s$ to the scoring function for
        $(z_1,\ldots,z_l)$
      \STATE set $g_y$ to the isotonic calibrator for
        \STATE \qquad $(s(x_1),y_1),\ldots,(s(x_l),y_l),(s(x),y)$
      \STATE set $p_y:=g_y(s(x))$
    \ENDFOR
  \end{algorithmic}
\end{algorithm}

Notice that Lemma~\ref{lem:finite-log} continues to hold for SVA predictors;
therefore, they never suffer infinite loss even under the log loss function.
On the other hand,
the following proposition shows that SVA predictors can violate the property~\eqref{eq:validity}
of unbiasedness in the large;
in particular, they are not Venn predictors
(cf.\ Corollary~\ref{cor:validity}).

\begin{proposition}\label{prop:counterexample}
  There exists a simplified Venn--Abers predictor
  violating~\eqref{eq:validity} for some $l$.
\end{proposition}

\begin{proof}
  Let the object space be the real line, $\mathbf{X}:=\mathbb{R}$,
  and the probability distribution %
  generating independent observations $(X,Y)$
  be such that:
  the marginal distribution of $X$ is continuous;
  the probability that $X>0$ (and, therefore, the probability that $X<0$) is $1/2$;
  the probability that $Y=1$ given $X<0$ is $1/3$;
  the probability that $Y=1$ given $X>0$ is $2/3$.
  Therefore, $\Prob(Y=1)=1/2$.
  Let $l$ be a large number
  (we are using a somewhat informal language, but formalization will be obvious).
  Given a training set $(z_1,\ldots,z_l)$, where $z_i=(x_i,y_i)$ for all $i$,
  the scoring function $s$ is:
  $$
    s(x)
    :=
    \begin{cases}
      0 & \text{if $x\in\{x_1,\ldots,x_l\}$ and $x<0$}\\
      1 & \text{if $x\in\{x_1,\ldots,x_l\}$ and $x>0$}\\
      2 & \text{if $x\notin\{x_1,\ldots,x_l\}$}.
    \end{cases}
  $$
  It is easy to see that, with high probability,
  $$
    \underline{V} \approx 2/3,
    \quad
    \overline{V} = 1.
  $$
  Therefore, \eqref{eq:validity} is violated.
\end{proof}

Proposition~\ref{prop:counterexample} shows that SVA predictors are not always valid;
however, the construction in its proof is artificial,
and our hope is that they will be ``nearly valid'' in practice,
since they are a modification of provably valid predictors.

For each dataset/classifier combination,
we repeat the same experiment a total of 100 times
for standard classifiers (denoted W in the tables), SVA, and DIR
and 16 times for VA (because of the computational inefficiency of the latter)
and average the results.
The same 100 random splits into training and test sets are used for W, SVA, and DIR,
but for VA the 16 splits are different.

Tables~\ref{tab:log}--\ref{tab:sq} compare the overall losses
computed according to \eqref{eq:MLE}
(MLE, used in Table~\ref{tab:log})
and \eqref{eq:RMSE}
(RMSE, used in Table~\ref{tab:sq})
for probabilities generated
by the standard classifiers as implemented in Weka (W)
and the corresponding Venn--Abers (VA),
simplified Venn--Abers (SVA), and direct isotonic-regression (DIR) predictors.
The values in bold indicate the lowest of the four losses for each dataset/classifier combination.
The column titles mention both fuller and shorter names for the six standard classifiers;
the short name ``SVM'' is especially appropriate when using VA, SVA, and DIR,
in which case the application of the sigmoid function in Platt's method is redundant.
The three entries of $\infty$ in the column W for logistic regression of Table~\ref{tab:log}
come out as infinities in our experiments only because of the limited machine accuracy:
logistic regression sometimes outputs probabilistic predictions
that are so close to 0 or 1 that they are rounded to 0 or 1, respectively, by hardware.

In the case of MLE,
the VA and SVA methods improve the predictive performance
of the majority of the standard classifiers on most datasets.
A major exception is J48 Bagging.
The application of bagging to J48 decision trees improves the calibration significantly
as bagging involves averaging over different training sets
in order to reduce the underlying classifier's instability.
The application of VA and SVA to J48 Bagging
is not found to improve the log or square loss significantly.
What makes VA and SVA useful is that for many data sets other classifiers,
less well calibrated than J48 Bagging,
provide more useful scores.

In the case of RMSE,
the application of VA and SVA also often improves probabilistic predictions.

\begin{table*}
\caption{The ranking of the best three methods (among W, VA, SVA, and DIR)
  for each dataset according to the two loss functions (see the text for details).}
\label{tab:summary}
\hspace*{-0.5cm}
\begin{tabular}{l|ll}
            & \multicolumn{1}{c}{log loss}  & \multicolumn{1}{c}{square loss} \\
\hline
Australian  & W (JB), VA (LR), SVA (LR)     & W (JB), SVA (JB), VA (LR) \\
Breast      & SVA (NB), VA (NB), W (JB)     & SVA (NB), VA (SVM), DIR (NB) \\
Diabetes    & VA (LR), SVA (SVM), W (SVM)   & VA (SVM), W (LR), SVA (SVM) \\
Echo        & VA (SVM), SVA (NB), W (JB)    & VA (NB), SVA (NB), W (SVM) \\
Hepatitis   & VA (SVM), SVA (NB), W (JB)    & SVA (NB), VA (NB), DIR (NB) \\
Ionosphere  & SVA (NB), VA (SVM), W (SVM)   & SVA (NB), DIR (NB), W (JB) \\
Labor       & SVA (SVM), W (NN), VA (SVM)   & W (NB), SVA (SVM), DIR (NB) \\
Liver       & VA (NN), W (JB), SVA (LR)     & VA (NN), SVA (NN), W (JB) \\
Vote        & SVA (SVM), W (SVM), VA (J)    & W (SVM), SVA (SVM), VA (J) \\
\end{tabular}
\end{table*}

Whereas in the case of square loss the DIR method
often produces values comparable to VA and SVA,
under log loss this method fares less well
(which is not obvious from \cite{zadrozny/elkan:2002},
which only uses square loss).
In all our experiments DIR suffers infinite log loss for at least one test observation,
which makes the overall MLE infinite.
There are modifications of the DIR method preventing probabilistic predictions in $\{0,1\}$
(such as those mentioned in \cite{niculescu-mizil/caruana:2012}, Section~3.3),
but they are somewhat arbitrary.

Table~\ref{tab:summary} ranks,
for each loss function and dataset,
the four calibration methods: W (none), VA (Venn--Abers), SVA (simplified Venn--Abers),
and DIR (direct isotonic regression).
Only the first three methods are given
(the best, the second best, and the second worst),
where the quality of a method is measured by the performance
of the best underlying classifier
(indicated in parentheses using the abbreviations
given in the column titles of Tables~\ref{tab:log}--\ref{tab:sq})
for the given method, data set, and loss function.
Notice that we are ranking the four calibration methods
rather than the 24 combinations of Weka classifiers with calibration methods
(e.g., were we ranking the 24 combinations,
the entry for log loss and Australian would remain the same
but the next entry, for log loss and Breast, would become
``SVA (NB), VA (NB), VA (LR)'').

For MLE, the best algorithm is VA or SVA for 8 data sets out of 9;
for RMSE this is true for 6 data sets out of 9.
In all other cases the best algorithm is W rather than DIR.
(And as discussed earlier,
in the case of log loss the performance of DIR is especially poor.)
Therefore, it appears that the most interesting comparisons are between W and VA
and between W and SVA.

What is interesting is that VA and SVA perform best on equal numbers of datasets,
4 each in the case of MLE and 3 each in the case of RMSE,
despite the theoretical guarantees of validity for the former method
(such as Theorem~\ref{thm:validity}).
The similar performance of the two methods needs to be confirmed
in more extensive empirical studies, but if it is,
SVA will be a preferable method because of its greater computational efficiency.

Comparing W and SVA,
we can see that SVA performs better than W on 7 data sets out of 9 for MLE,
and on 5 data sets out of 9 for RMSE.
And comparing W and VA,
we can see that VA performs better than W on 6 data sets out of 9 for MLE,
and on 5 data sets out of 9 for RMSE.
This suggests that SVA might be an improvement of VA not only in computational
but also in predictive efficiency (but the evidence for this is very slim).

\section{Conclusion}

This paper has introduced a new class of Venn predictors
thereby extending the domain of applicability of the method.
Our experimental results suggest that the Venn--Abers method
can potentially lead to better calibrated probabilistic predictions
for a variety of datasets and standard classifiers.
The method seems particularly suitable in cases
where alternative probabilistic predictors produce overconfident
but erroneous predictions
under an unbounded loss function such as log loss.
In addition,
the results suggest that an alternative simplified Venn--Abers method
can yield similar results while retaining computational efficiency. 

Unlike the previous methods for improving
the calibration of probabilistic predictors,
Venn--Abers predictors enjoy theoretical guarantees of validity
(shared with other Venn predictors).

\subsection*{Acknowledgments}

Thanks to the reviewers for helpful comments,
which prompted us to state explicitly Proposition~\ref{prop:counterexample}
and add several clarifications.
In our experiments in Section~\ref{sec:experiments}
we used the R language \cite{r};
in particular, we used the implementation of the PAVA in the standard R package \texttt{stats}
(namely, the function \texttt{isoreg}).
The first author has been partially supported by EPSRC (grant EP/K033344/1).

\end{document}